\documentclass[conference]{IEEEtran}
\IEEEoverridecommandlockouts
\usepackage{cite}
\usepackage{amsmath,amssymb,amsfonts}
\usepackage{algorithmic}
\usepackage{graphicx}
\usepackage{textcomp}
\usepackage{multirow}
\usepackage{xcolor}
\usepackage{amsthm}
\usepackage{adjustbox}
\usepackage[english]{babel}
\usepackage[T1]{fontenc}
\usepackage[utf8]{inputenc}

\newtheorem{theorem}{Theorem}
\def\BibTeX{{\rm B\kern-.05em{\sc i\kern-.025em b}\kern-.08em
    T\kern-.1667em\lower.7ex\hbox{E}\kern-.125emX}}
\begin{document}

%\title{Learning Graphs via Short Random Walks Originated from Dense Vertices \\
\title{Can NetGAN be improved on short random walks?\\
%{\footnotesize \textsuperscript{*}Note: Sub-titles are not captured in Xplore and should not be used}
}

\author{\IEEEauthorblockN{Jalilifard A.; Caridá V.F.; Mansano A.F.; Cristo R.S.}
\IEEEauthorblockA{\textit{Data Science Team - Digital Customer Service} \\
\textit{Itaú Unibanco, São Paulo, Brazil}\\
amir.jalilifard; vinicius.carida; alex.mansano; rogers.cristo}@itau-unibanco.com.br
}

\maketitle

\begin{abstract}
Graphs are useful structures that can model several important real-world problems. Recently, learning graphs has drawn considerable attention, leading to the proposal of new methods for learning these data structures. One of these studies produced NetGAN, a new approach for generating graphs via random walks. Although NetGAN has shown promising results in terms of accuracy in the tasks of generating graphs and link prediction, the choice of vertices from which it starts random walks can lead to inconsistent and highly variable results, specially when the length of walks is short. As an alternative to a random starting, this study aims to establish a new method for initializing random walks from a set of dense vertices. We purpose estimating the importance of a node based on the inverse of its influence over the whole vertices of its neighborhood through random walks of different sizes. The proposed method manages to achieve a significantly better accuracy, less variance and lesser outliers.
\end{abstract}


\section{Introduction}
One of the main obstacles in artificial intelligence applications is to discover valuable hierarchical models that represent probability distributions over the types of data encountered~\cite{Bengio2009learning}. The most remarkable results in deep learning have involved discriminative models, often used to map data into high-dimensional, rich sensory input to a class label~\cite{Hinton2012,Krizhevsky2012imagenet}.

Recently, there has been renewed interest in evaluating graphs with machine learning due to a collection of practical applications in which they can be employed. Nonetheless, it is quite challenging to obtain a model that captures all the essential properties of real graphs. As a particular non-Euclidean data structure for machine learning, graph analysis generally concentrates on node classification, link prediction, and clustering~\cite{Zhou2018graph}.

Graphs can be utilized to describe numerous processes across diverse fields, such as social sciences (social networks)~\cite{Hamilton2017inductive,Kipf2016semi}, natural sciences (physical systems~\cite{Sanchez2018graph,Battaglia2016interaction} and protein-protein interaction networks~\cite{Fout2017protein}), knowledge graphs~\cite{Hamaguchi2017knowledge} as well as other researching areas~\cite{Khalil2017learning}.

Regarding graph learning, Graph Neural Networks (GNN)~\cite{Gori2005, Scarselli:2009:GNN:1657477.1657482} are deep learning methods which have been widely applied due to its convincing performance and high interpretability.

As a result of the complications on approximating many intractable probabilistic computations that arise in maximum likelihood evaluation and related strategies, and also due of the difficulty of leveraging the benefits of piece-wise linear units in the generative context, Goodfellow~\cite{Goodfellow2016NIPS2T} has proposed Generative Adversarial Networks (GAN) as a new generative model estimation procedure.

GANs promoted significant advancements in the state-of-the-art over the classic prescribed approaches like Gaussian mixtures~\cite{Blanken2007multimedia}. The method also achieved great results in other scenarios such as image generation and 3D objects synthesis~\cite{Karras2018progressive,Berthelot2017BEGAN,Jiajun2016NIPS}.

Associating the concepts of both GNN and GAN, NetGAN~\cite{Bojchevski2018netgan} was proposed as one of the first methods to produce neural graph generative models. The essential idea behind it is to convert the problem of graph generation into walk generation, employing random walks from a specific graph as data input, and training a generative model using the GAN architecture. The generated graph tends to preserve important topological features of the original graph, including the initial amount of nodes~\cite{Zhou2018graph}.

NetGAN approach offers strong generalization features, as indicated by its competitive link prediction performance on several data-sets. It can further be used for generating graphs well-suited to capture the complex nature of real-world networks~\cite{Bojchevski2018netgan}. On the other hand, in spite of the above-mentioned aspects, the method scalability is related as one of its drawbacks. NetGAN takes numerous generated random walks to generate representative transition counts for large graphs. Accordingly, instead of determining an arbitrary amount of random walks, a probable enhancement of NetGAN would be the adoption of a conditional generator capable of, given a starting node, provide a more even graph walking coverage and better scalability.

In this work we introduce a method to learn graphs through short random walks originated from dense vertices. We specifically emphasis on using short random walks because as the size of random walks is increased, the proportion of number of separate observations \begin{math} n \end{math} to the number of predictors \begin{math} p \end{math} tends to decrease, resulting in over-fitting problem. This is more evident when the training graph has small size. In order to be able to use short random short walks, it is necessary to start them from the vertices which provide more information for the learning algorithm. Thus, we propose a new procedure to compute the importance of a node based on the inverse of its influence over all vertices in its neighborhood, through random walks of varied sizes. Ultimately, the most important nodes are used as the starting points.

\section{Problem statement}\label{sec:problem}
NetGAN initializes its random walks from a set of randomly chosen nodes. Although this approach is appropriate for sampling from different regions of the graph, it also leads to having different results with high variance after each training session. This becomes more noticeable for random walks of short length, producing less generalized trained models. We propose a method for determining the best nodes for initializing random walks instead of choosing them randomly. By limiting the possible start points and adopting better random walk initializers, our method reduces the variability of results and enhances the learning accuracy.

A graph is denoted as \begin{math} G = (V,E,W) \end{math}, where \begin{math} V \end{math} is the set of vertices, \begin{math} E \end{math} represents the edges between the vertices and \begin{math} W \end{math} is the weight of edges. The problem of finding adequate start points for initializing random walks is to find a set of vertices that decrease the entropy of walks which itself results in a higher information gain. A desired set of vertices should achieve a better graph generation accuracy and a training model which is less dependent on the random choice of initializers.

In the problem of finding best random walk initializers, there exist two main issues: (1) measuring the influence of a vertex on a set of neighbors through random walks, and (2) converting this influence to a metric which determines the adequacy of a vertex for initializing a random walk. In the next session, we will explain our proposed method for tackling each of these issues. 

\section{Method}\label{sec:method}
Good initializing points are important for learning the structure of a graph through random walks. In a scenario when the beginning vertex has limited access to its neighborhood through different paths, the amount of information gained while visiting the nodes through distinct random walks is not significant. Instead of selecting the starting point randomly, we follow the motivation of~\cite{Hinneburg1998} for initializing the random walks from graph centroids from density point of view. We assume that if a vertex has several access to a group of graph vertices through \begin{math} l-steps \end{math} of random walks, this vertex is located in a dense region of graph, providing better information for an algorithm like NetGAN. On the other hand, if a vertex is in a sparse region, it has access to less number of regions and as a result supplies less information.

We follow the definition of neighborhood distance proposed in~\cite{Zhou2009}, although, instead of using this distance for finding the closeness of two nodes, we take advantage of the probabilistic characteristic of this distance so as to find the dense regions of the graph. First, the transition probability matrix is calculated as following:

\begin{equation} \label{eq:1}
    P_{v_{i},v_{j}} = \begin{cases}\frac{w_{v_{i},v_{j}}}{w_{v_{i},v_{1}} + w_{v_{i},v_{2}}+ ... + w_{v_{i},v_{m}}} & (V_{i},V_{m}) \epsilon E_{G} \\0 & otherwise \end{cases}
\end{equation}

where \begin{math} w_{v_{i},v_{m}} \end{math} is the weight of the edge between \begin{math} v_{i} \end{math} and \begin{math} v_{m} \end{math} and \begin{math} m \end{math} is the number of neighbors of \begin{math} v_{i} \end{math}. Here, we assume that all the weights are equal to 1, although the proposed method is applicable for weighted graphs.

Having the transition probability, the distance between two nodes is defined as

\begin{equation} \label{eq:2}
    d_{v_{i},v_{m}} = \sum_{l:v_{i}\rightsquigarrow v_{m}} p(l)c(1-c)^{l}
\end{equation}
 where \begin{math} l \end{math} is the length of random walk path from \begin{math} v_{i} \end{math} to \begin{math} v_{m} \end{math}, \begin{math} c \end{math} is the probability of returning to the initial state and \begin{math} p \end{math} is the probability of reaching from \begin{math} v_{i} \end{math} to \begin{math} v_{m} \end{math}.

Then the neighborhood random walk distance matrix on a structure graph is
\begin{equation} \label{eq:3}
    R_{v_{i},v_{m}} = \sum_{\eta=0}^l P_{v_{i},v_{m}}^{\eta}c(1-c)^{\eta} 
\end{equation}

\begin{theorem}
Given a graph \begin{math} G \end{math} and a set of vertices \begin{math} V \end{math}, if vertex \begin{math} v_{i} \end{math} is located in a denser region than \begin{math} v_{j} \end{math}, initiating a random walk from the first results in a lower entropy than the later, carrying more information.
\end{theorem}

\begin{proof}
The entropy of an event is defined as

\begin{equation}
    S = - \sum_{i} P_{i}\log P_{i}
\end{equation}

For a sub-graph \begin{math} S = G(V_{s}, E_{s}) \end{math} the density of \begin{math} S \end{math} is defined to be

\begin{equation} \label{eq:8}
d(S) = \frac{\mid E_{s} \mid}{\mid V_{s} \mid}    
\end{equation}

If vertex \begin{math} v_{i} \end{math} is located in a denser sub-graph than another vertex, say \begin{math} v_{j} \end{math}, \begin{math} v_{i} \end{math} has higher degree than \begin{math} v_{j} \end{math}. Based on the equations \ref{eq:1} and the definition of density function in \cite{Zhou2009}, denser nodes have access to more other vertices through distinct random walks, meaning that the sum of all degrees in a dense region is higher than the sparse one. Therefore, if \begin{math} m \end{math} and \begin{math} n \end{math} are number of neighbors of \begin{math} v_{i} \end{math} and \begin{math} v_{j} \end{math} respectively, therefore
\begin{equation} \label{eq:9} w_{v_{i},v_{1}} + ... + w_{v_{i},v_{m}} > w_{v_{j},v_{1}} + ... + w_{v_{j},v_{n}} 
    \Rightarrow P_{v_{i},v_{m}} < P_{v_{j},v_{n}}
\end{equation}

From equation (\ref{eq:2}) we defined the distance between two nodes as sum of the transition probabilities of a random walk through all the middle edges. Thereby, considering (\ref{eq:9})

 \begin{equation} \label{eq:3}
    \sum_{\eta=0}^l c(1-c)^{\eta} \times \prod_{k=0}^m P_{v_{i},v_{k}} <  \sum_{\eta=0}^l c(1-c)^{\eta} \times \prod_{k^{'}=0}^n P_{v_{i},v_{k^{'}}}
\end{equation}

The influence of vertex \begin{math} v_{i} \end{math} on another vertex \begin{math} v_{m} \end{math} is defined as

\begin{equation} \label{eq:4}
    f_{B}^{v_{m}}(v_{i}) = 1 - e^{ - \frac{d_{v_{i},v_{m}}}{2\sigma^2}}
\end{equation}
 where \begin{math} f_{B}^{v_{m}}(v_{i}) \epsilon \space [0,1] \end{math}.
 
\vspace{1cm}

 The density function is calculated as following
 \begin{equation} \label{eq:5}
    f_{B}^{D}(v_{i}) = \sum_{v_{i} \epsilon V } 1 - e^{ - \frac{d_{v_{i},v_{m}}}{2\sigma^2}}
\end{equation}

Based on equation (\ref{eq:4}), the influence metric is proportional to the random walk distance between two nodes. Then the lesser the  sum of all the random walk distances of \begin{math} v_{i} \end{math} the denser it is. Since we defined the distance as the probability of successive walks, the entropy of a random walk is calculated as 

\begin{equation}
    S = - \sum_{i} f_{B}^{D}(v_{i})\log f_{B}^{D}(v_{i})
\end{equation}

Consequently, the denser a vertex, the lower the distance and entropy and as result the more information gain.
\end{proof}

Notice that by accessibility we mean the probability of walking on different paths through a random walk. If a node is located on a sparse region, the probability of walking on the same randomly chosen path is more than a denser region. Also, we defined the density in terms of accessibility through random walks of length \begin{math} l \end{math}. Without this constraint, the theory would be wrong (consider a dense node in a small component that is disconnected from a large component or perhaps connected to a large component via a single link and consider a sparse node in the large component which is connected to a node in a dense region).

\section{Results and discussion}\label{sec:results}
In this section, we evaluate the quality of generated graphs via random walks started from dense nodes in terms of link prediction precision and ROC curve. We compare all the results with those of NetGAN. For all the experiments, we followed all the preprocessing steps of the original paper. The CORA-ML \cite{mccallum2000automating} was used in all the conducted experiments. As in~\cite{Bojchevski2018netgan}, we treat the graph as undirected. We randomly used one of the connected components of CORA-ML with 303 vertices and tested our algorithm against the performance of original NetGAN when applied on the same graph.

In order to find the best nodes to start a random walk, we first sampled 100 random walks from each vertex in order to find the graph paths. After forming the transition and the distance matrix, the densities were calculated and then the vertices were ordered in ascending fashion. Since NetGAN's authors did not inform the motivation of the choice of the configurations they used for analyzing the results, we first tested the original NetGAN with diverse configurations. We found out that the choice of batch size, length of random walks and the initial vertices of random walk can significantly change the results. The more the batch size and the higher the length of the random walk, the higher the precision. 

Although there is a high correlation between the performance of NetGAN with batch size and random walk length separately, we could not find any meaningful correlation between batch size and random walk directly. This comes from the fact that based on our tests setting a high random walk length can compensate the small size of batch and vice-versa. Furthermore, when the batch size is chosen to be small, the choice of initial random walk points becomes more crucial. On the other hand, if both random walk length and batch size are set to be big, this may cause over-fitting. Thereby, we chose the configurations that are proportional to the tests carried out by authors of NetGAN (see Figures~\ref{edge_Overlap}~and~\ref{precision_ROC}).

\begin{figure}[!ht]
  \centering
  \includegraphics[width=9cm]{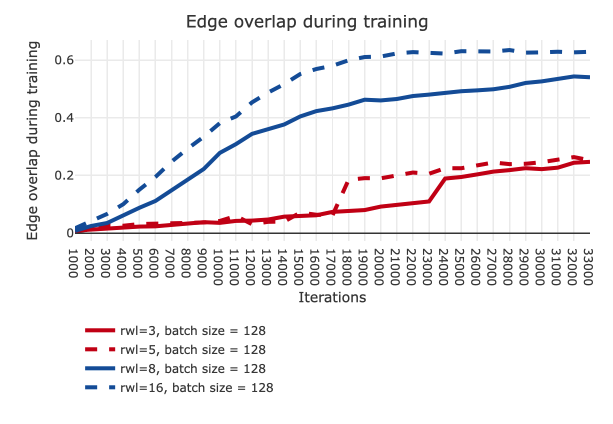}
  \caption{Edge overlap during the training for random walks of different lengths.}
  \label{edge_Overlap}
\end{figure}

\begin{figure}[!ht]
  \centering
  \includegraphics[width=9cm]{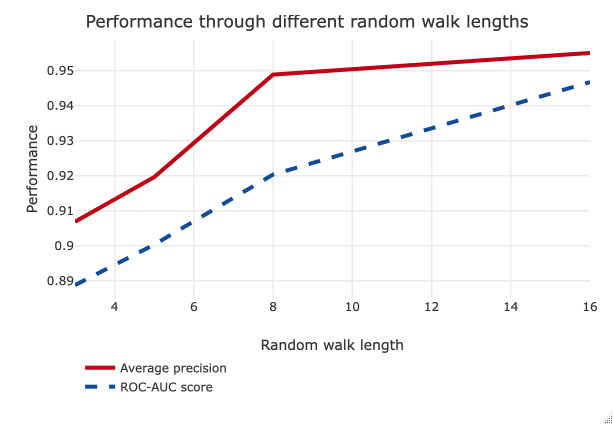}
  \caption{The effect of random walk length on average precision and ROC-AUC score}
  \label{precision_ROC}
\end{figure}

In order to evaluate the proposed method, we calculated the vertex density for random walks of length equal to 2 to 4 and batch size of 13, 19, 25, respectively. As it was mentioned before, these configurations are proportional to the original configuration used by authors in~\cite{Bojchevski2018netgan}.

Unlike the random approach proposed in~\cite{Bojchevski2018netgan}, our approach limits the choice of vertices to those with a higher chance of representing new information through each random walk. As seen in Figure~\ref{boxplot_config}, for all the configurations, the proposed method yields a higher accuracy and lesser variance for the most part. As illustrated, the original method has a poor performance for short random walks, although as the length of walks is increased, the accuracy of original method becomes closer to our method. 

\begin{figure}[ht]
  \centering
  \includegraphics[width=9cm]{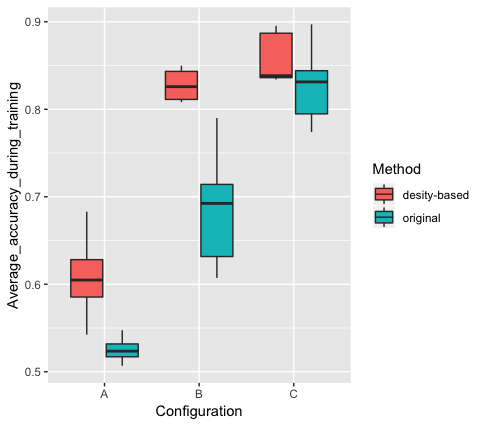}
  \caption{The accuracy for both randomly and density based chosen initial vertices for several configurations (A: batch size = 13, random walk length = 2; B: batch size = 19, random walk length = 3; C: batch size = 25, random walk length = 4;)}
  \label{boxplot_config}
\end{figure}

As shown in Figure~\ref{overlap_boxplot_config}, our method also results in a better edge overlap during the training and a less average variance through a variety of configurations.

\begin{figure}[ht]
  \centering
  \includegraphics[width=9cm]{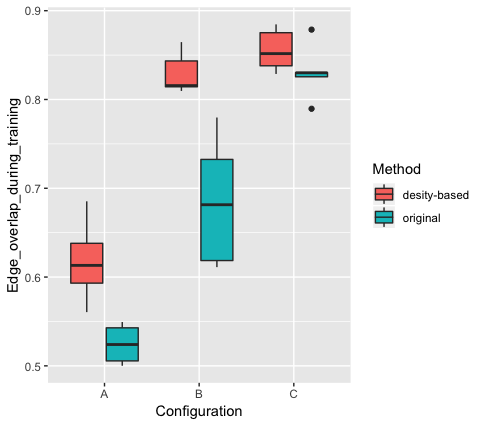}
  \caption{The edge overlap for both randomly and density based chosen initial vertices for several configurations(A: batch size = 13, random walk length = 2; B: batch size = 19, random walk length = 3; C: batch size = 25, random walk length = 4;)}
  \label{overlap_boxplot_config}
\end{figure}

We finally investigate the performance of our method in terms of ROC-AUC score and average precision score. Based on our results, the proposed method has better average link prediction precision and ROC-AUC scores (see Figure~\ref{ROC_boxplot_config}).

\begin{figure}[ht]
  \centering
  \includegraphics[width=9cm]{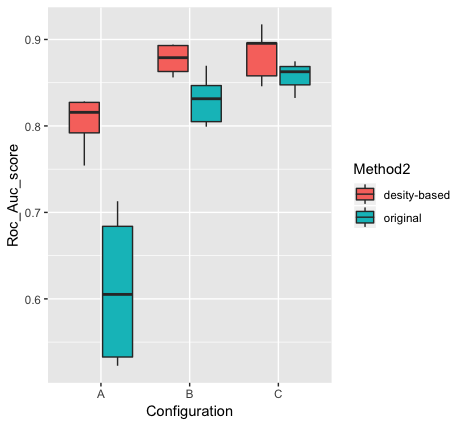}
  \caption{The link ROC-AUC score for both randomly and density based chosen initial vertices for several configurations (A: batch size = 13, random walk length = 2; B: batch size = 19, random walk length = 3; C: batch size = 25, random walk length = 4;)}
  \label{ROC_boxplot_config}
\end{figure}

\begin{figure}[ht]
  \centering
  \includegraphics[width=9cm]{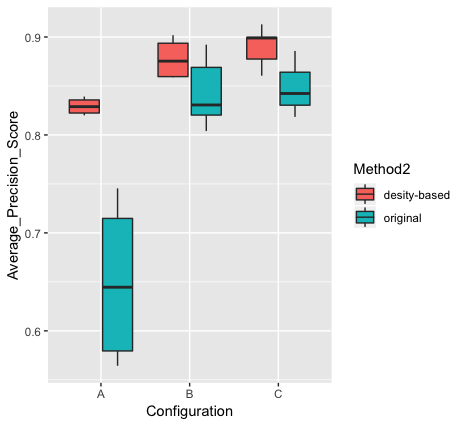}
  \caption{The link prediction precision for both randomly and density based chosen initial vertices for several configurations (A: batch size = 13, random walk length = 2; B: batch size = 19, random walk length = 3; C: batch size = 25, random walk length = 4;)}
  \label{precision_boxplot_config}
\end{figure}

\begin{table*}[ht!]
\centering
\caption{The performance of NetGAN in terms of average accuracy during training, average link prediction and average ROC-AUC scores for different three different configurations of dense and random initializers.}
%\resizebox{\linewidth}{!}{%
\begin{adjustbox}{width=1\textwidth,center}
\begin{tabular}{cccccccc} 
\hline
\\[-0.5em]
\multirow{2}{*}{\textbf{Batch size} } & \multirow{2}{*}{\textbf{Random walk length} } & \multicolumn{2}{c}{\textbf{Average training accuracy} } & \multicolumn{2}{c}{\textbf{Average precision} } & \multicolumn{2}{c}{\textbf{Average ROC-AUC} }
\\ 
& & \textit{Random} & \textit{Dense} & \textit{Random} & \textit{Dense} & \textit{Random} & \textit{Dense}
\\
\\[-0.5em]
\hline
\\[-0.5em]
13 & 2 & 0.52 & 0.60 & 0.64 & 0.82 & 0.61 & 0.80
\\
19 & 3 & 0.69 & 0.82 & 0.83 & 0.87 & 0.82 & 0.87
\\
25 & 4 & 0.82 & 0.86 & 0.85 & 0.89 & 0.85 & 0.88
\\
40 & 5 & 0.88 & 0.90 & 0.89 & 0.89 & 0.86 & 0.85
\\[-0.1em]
\hline
\end{tabular}
\end{adjustbox}
%}
\label{tab:table3}
\end{table*}

As shown by several different metrics, starting random walks from dense vertices lead having better performance in NetGAN. Also as illustrated in Table~\ref{tab:table3}, using dense vertices result in significantly better average accuracy during the training phase, average link prediction precision and average ROC-AUC scores. Using the current method, we managed to get link prediction precision up to 92.8\% and an average score of 89\%. Although the objective of the proposed method is to improve the performance of NetGAN for short random walks, we compared the performance of our method with the random initializer approach. Based on our results, when the size of random walks is increased (when the random walk length becomes bigger than 6), the performance of the random approach in terms of average ROC-AUC and average link prediction precision is slightly better compared to our method. Nevertheless, the performance of our method in terms of average accuracy during the training phase is always better independent of the size of the random walk.

Although this was not our intention in this research, there are two possible ways to improve the performance of the current method for random walks of even larger size than 6. As we mentioned earlier, we sampled the graph using 100 random walks of size 8 for each vertices in order to find the possible paths of the graph. This means that our sample is just an approximation of the real graph paths and finding all the paths of size \begin{math} l \end{math} between every two nodes may return a slightly better set of dense vertices and improve the performance of current method for longer random walks. The second possible way is to add a percentage of randomly chosen vertices to the set of dense vertices and also choosing the random walk paths based on the importance of each vertex in terms of its density (for now the probability of each vertex to be chosen is equal for all the vertices existing in the set of \begin{math} n \end{math} most dense vertices). 

\section{Conclusion}\label{sec:conclusion}
NetGAN uses a random strategy to choose vertices that initialize the random walks. This approach raises the variance in accuracy and precision of link prediction and compels the use of more steps in the random walk to compensate the possible poor initializers, thus getting more information through larger walks.

In this paper we develop a new method dedicated to determine better starting points so that NetGAN can get less variant results and more precise predictions by learning graphs via short random walks. Applying the probabilistic distances, we demonstrated that if the random walks starts from denser regions they may have lower entropy, culminating in more information gain. All the vertices are ordered based on its calculated density value. Since denser vertices have an abundance of connections, starting a random walk from those nodes minimizes the chance of reiterating through the same path, and consequently provides more information for NetGAN.

We tested our hypothesis with various configurations and with multiple performance metrics including accuracy, ROC-AUC score, edge overlap during the training and average link prediction precision. Compared with NetGAN random approach, our method had better results in all experiments. More specifically, our results show that starting the random walk from dense vectors significantly increases the accuracy and link prediction precision. Using short random walks not only decreases the training time, but unlike the large random walks it makes the trained model less prone to over-fitting. Nonetheless, the performance of the random strategy becomes closer to our method as the length of random walks increases. One possible future work could concentrate on the investigation of combining some randomly determined vertices with the dense ones in order to have even better sampling strategy.

\section{Conflict of interest}
The current method was proposed and tested by a group of data scientists from Itaú Unibanco. Any opinions, findings, and conclusions expressed in this manuscript are those of the authors and do not necessarily reflect the views, official policy or position of Itaú Unibanco.

%Itaú Unibanco

\bibliography{bib.bib}{}

%\bibliography{bib.bbl}{}

%\bibliography{bib.txt}{}

%\begin{thebibliography}

%\end{thebibliography}

\bibliographystyle{IEEEtran}

\vspace{12pt}

\end{document}